\documentclass{article}
\usepackage{authblk}
\usepackage[margin=1in]{geometry}
\pdfoutput=1

\usepackage{times}

\usepackage{graphicx} % more modern

\usepackage{subfigure} 

\usepackage[numbers]{natbib}

\usepackage{hyperref}

\usepackage{amsfonts,amsmath,amsthm,amscd,dsfont,mathrsfs,amssymb}
\usepackage{fullpage}
\usepackage{url}            % simple URL typesetting
\usepackage{booktabs}       % professional-quality tables
\usepackage{nicefrac}       % compact symbols for 1/2, etc.
\usepackage{microtype}      % microtypography

\usepackage{verbatim}
\usepackage{algorithmic, algorithm}
\usepackage{amssymb}
\usepackage{epsfig}
\usepackage{wrapfig}
\newtheorem{defin}{Definition}
\newtheorem{rem}{Remark}

\newcommand{\rank}{\text{rank}}

\newcommand{\pnorm}[2]{{\Vert #1 \Vert} _{#2}}

\newcommand{\numobs}{n}
\newcommand{\inprod}[2]{\langle#1, #2\rangle}

\newcommand{\reals}{\mathbb{R}}

\mathchardef\mhyphen="2D

\DeclareMathOperator*{\argmax}{arg\,max}

\newcommand{\vertiii}[1]{{\left\vert\kern-0.25ex\left\vert\kern-0.25ex\left\vert #1
    \right\vert\kern-0.25ex\right\vert\kern-0.25ex\right\vert}}

\renewcommand\exp{\operatorname{exp}}

\newcommand{\vect}[1]{{\boldsymbol{#1}}}
\def\balpha{\vect{\alpha}}

\def\bDelta{\vect{\Delta}}

\def\bu{{\mathbf{u}}}
\def\bv{{\mathbf{v}}}

\def\bx{{\mathbf{x}}}
\def\by{{\mathbf{y}}}

\def\bA{{\mathbf{A}}}
\def\bB{{\mathbf{B}}}

\def\bD{{\mathbf{D}}}

\def\bH{{\mathbf{H}}}

\def\bL{{\mathbf{L}}}
\def\bM{{\mathbf{M}}}

\def\bU{{\mathbf{U}}}
\def\bV{{\mathbf{V}}}

\def\bX{{\mathbf{X}}}
\def\bY{{\mathbf{Y}}}

\def\bold0{{\mathbf{0}}}

\def\bbR{{\mathbb{R}}}

\def\cA{\mathcal{A}}

\def\cN{\mathcal{N}}

\def\cU{\mathcal{U}}
\def\cV{\mathcal{V}}

\def\sfr{{\mathsf{r}}}
\def\sfs{{\mathsf{s}}}

\def\sfA{\mathsf{A}}
\def\sfB{\mathsf{B}}

\def\sfL{\mathsf{L}}

\def\sfS{\mathsf{S}}

\def\sfU{\mathsf{U}}

\newtheorem{theorem}{Theorem}
\newtheorem{lemma}{Lemma}
\renewcommand{\text}[1]{{\textnormal{#1}}}

\newcommand{\dima}{n}
\newcommand{\dimb}{d}

\newcommand{\matsnorm}[2]{|\!|\!| #1 | \! | \!|_{{#2}}}

\newcommand{\nuclear}[1]{\ensuremath{\matsnorm{#1}{\operatorname{\tiny{nuc}}}}}
\newcommand{\nucnorm}[1]{\ensuremath{\nuclear{#1}}}
\newcommand{\order}{{\Theta}}

\newcommand{\kbar}{\ensuremath{\bar{k}}}
\newcommand{\betal}{\ensuremath{\bB^{(\sfL)}}}
\newcommand{\betals}{\ensuremath{\bB^{(\sfL \cup \sfS)}}}
\newcommand{\Mo}{\ensuremath{\tilde{M_1} }}
\newcommand{\mkb}{\ensuremath{m_{\kbar}}} %%RSC constant
\newcommand{\defn}{\ensuremath{:  =}}
\DeclareMathOperator{\Sspan}{span}
\newcommand{\func}{\textsc{Greedy}}
\newcommand{\geco}{\textsc{GECO}}
\newcommand{\greedysel}{\texttt{GreedySel}}
\newcommand{\OMPsel}{\texttt{OMPSel}}
\author[1]{Rajiv Khanna}
\author[1]{Ethan R. Elenberg}
\author[1]{Alexandros G. Dimakis}
\author[2]{Sahand Negahban}
\affil[1]{Department of Electrical and Computer Engineering \authorcr The University of Texas at Austin \authorcr \texttt{\{rajivak,\,elenberg\}@utexas.edu}, \texttt{dimakis@austin.utexas.edu}}
\affil[2]{Department of Statistics \authorcr Yale Univeristy \authorcr \texttt{sahand.negahban@yale.edu}}

\begin{document}

\title{On Approximation Guarantees for Greedy Low Rank Optimization}

\maketitle

\begin{abstract}
We provide new approximation guarantees for greedy low rank matrix estimation under standard assumptions of restricted strong convexity and smoothness. Our novel analysis  also uncovers previously unknown connections between the low rank estimation and combinatorial optimization, so much so that our bounds are reminiscent of corresponding approximation bounds in submodular maximization. Additionally, we also provide statistical recovery guarantees. Finally, we present empirical comparison of greedy estimation with established baselines on two important real-world problems. 
\end{abstract}
\section{Introduction}
\label{sec:introduction}
Low rank matrix estimation stands as a major tool in modern
dimensionality reduction and unsupervised learning. The singular value
decomposition can be used when the optimization objective is
rotationally invariant to the parameters. However, if we wish to optimize
over more complex objectives we must choose to either
optimize over the non-convex space (which have seen recent theoretical
success)~\cite{Parketal16s,Jainetal13s,ChenWa15s,LeeBres13,JaiTewKar14} or rely on convex
relaxations to the non-convex
optimization~\cite{Recht2010s,Neghaban:2011lowrank,RohTsy2011}.

More concretely, in the low-rank matrix optimization problem we wish
to solve
\begin{equation}
  \label{eq:lowrankopt}
  \argmax_{\Theta} \ell (\Theta) \quad \text{s.t. $\rank(\Theta) \leq r$} ,
\end{equation}
and rather than perform the computationally intractable optimization
above researchers have studied convex relaxations of the form
\begin{equation*}
  \label{eq:nucopt}
  \argmax_{\Theta} \ell (\Theta) - \lambda \nucnorm{\Theta} .
\end{equation*}
Unfortunately, the above optimization can be computationally
taxing. General purpose solvers for the above optimization problem
that rely on semidefinite programming require $\order(\dima^3 \dimb^3)$ computation, which is prohibitive. Gradient descent
techniques require $\order(\epsilon^{-1/2} (\dima^3 + \dimb^3))$
computational cost for an epsilon accurate solution. This improvement
is sizeable in comparison to SDP solvers. Unfortunately, it is still
prohibitive for large scale matrix estimation.

To alleviate some of the computational issues an alternate vein of research has focused on directly optimizing the
non-convex problem in \eqref{eq:lowrankopt}. To that end,
authors have studied the convergence properties of
\begin{equation*}
  \label{eq:altmin}
  \argmax_{\bU \in \reals^{\dima \times r}, \bV \in \reals^{\dimb \times r}} \ell  (\bU \bV^T) .
\end{equation*}
Solving the problem above automatically forces the solution to be
low rank and recent results have shown promising behavior. An
alternative approach is to optimize via rank one updates to the
current estimate~\cite{ShalevShwartz:2011vi,Wangetal2015}. This
approach has also been studied in more general contexts such as boosting~\cite{BuhYu2009},
coordinate descent~\cite{jaggi13FW,Jaggi:2010nn}, and incremental atomic norm optimization~\cite{Gribonval:2006ch,Barron08,khaetal16,Raoetal2015}.
\subsection{Set Function Optimization and Coordinate Descent}
\label{subsec:set}
The perspective that we take is treating low rank matrix estimation as
a set optimization over an infinite set of atoms. Specifically, we
wish to optimize
\begin{equation*}
%  \argmax_{\cV  \subset \cA : |\cV| \leq k} f(\sum_{v \in \cV} \alpha_v v) ,
  \argmax_{\{\bX_1, \ldots \bX_k\} \in \cA } \ell \left(\sum_{i=1}^{k} \alpha_i \bX_i\right),
\end{equation*}
where the set of atoms $\cA$ is the set of all rank one matrices with
unit operator norm. This settings is analogous to that taken in the
results studying atomic norm optimization, coordinate descent
via the norm in total variation, and Frank-Wolfe style
algorithms for atomic optimization. This formulation allows us to
connect the problem of low rank matrix estimation to that of
submodular set function optimization, which we discuss in the
sequel. Before proceeding we discuss related work and an informal
statement of our result.
\subsection{Informal Result and Related Work}
\label{subsec:informal}
Our result demonstrates an exponential decrease in the amount of error
suffered by greedily adding rank one matrices to the low rank matrix
approximation.
\begin{theorem}[Approximation Guarantee, Informal]
  If we let $\Theta_k$ be our estimate of the rank $r$ matrix
  $\Theta^*$ at iteration $k$, then for some universal constant $c$
  related to the restricted condition number of the problem we have
  \begin{equation*}
    \ell (\Theta_k) - \ell (0) \geq (1-\exp(-c k/r)) ( \ell (\Theta^*) - \ell (0)).
  \end{equation*}
\end{theorem}
Note that after $k$ iterations the matrix $\Theta_k$ will be at most
rank $k$. Now, we can contrast this result to related work.

\paragraph{Related work:} There has been a wide array of studies
looking at the computational and statistical benefits of rank one
updates to estimating a low rank matrix. At its most basic, the
singular value decomposition will keep adding rank one approximations
through deflation steps. Below we discuss a few of the results.

The work can be generally segmented into to sets of results. Those
results that present sublinear rates of convergence and those that
obtain linear rates. Interestingly, parallel lines of work have also
demonstrated similar convergence bounds for more general atomic or
dictionary element
approximations~\cite{BuhYu2009,Gribonval:2006ch,Barron08,khaetal16}. For
space constraints, we will summarize these results into two categories
rather than explicitly state the results for each individual paper.

If we define the atomic norm of a matrix $\bM \in \reals^{\dima \times \dimb}$ to be $\nucnorm{\bM}$ to be
the sum of the singular values of that matrix, then the bounds
establish in the sublinear convergence cases behave as
\begin{equation*}
  \ell(\Theta^*) - \ell(\Theta_k) \leq \frac{\nucnorm{\Theta^*}^2}{k} ,
\end{equation*}
where we take $\Theta^*$ to be the best rank $k$ solution. What we
then see is convergence towards the optimal bound. However,
we expect our statistical error to behave as
$r (\dima + \dimb)/\numobs$ where $\numobs$ is the
number of samples that we have received from our statistical model and
$\Theta^*$ is rank $r$~\cite{Neghaban:2011lowrank,RohTsy2011}. We can take $\nucnorm{\Theta^*} \approx r$,
which would then imply that we would need $k$ to behave as $\numobs/(\dima +
\dimb)$. However, that would then imply that the rank of our matrix
should grow linearly in the number of observations in order to achieve
the same statistical error bounds. The above error bounds are
``fast.'' If we consider a model that yields slow error bounds, then
we expect the error to behave like $\nucnorm{\Theta^*}
\sqrt{\frac{\dima + \dimb}{\numobs}}$. In that case, we can take $k
\geq \nucnorm{\Theta^*} \sqrt{\frac{\numobs}{\dima + \dimb}}$, which
looks better, but still requires significant growth in $k$ as a
function of $n$.

To overcome the above points, some authors have aimed to study similar
greedy algorithms that then enjoy exponential rates of convergence as
we show in our paper. These results share the most similarities with
our own and behave as
\begin{equation*}
  \ell (\Theta_k) \geq (1-\gamma^k) \ell (\Theta^*)
\end{equation*}
where $\Theta^*$ is the best over all set of parameters. This result
decays exponentially. However, when one looks at the behavior of
$\gamma$ it will typically act as $\exp{(   \nicefrac{-1}{\min({\dima,\dimb})} )}$, for an $\dima
\times \dimb$ matrix. As a result, we would need to take $k$ on the
order of the number of the dimensionality of the problem in order to
begin to see gains. In contrast, for our result listed above, if we
seek to only compare to the best rank $r$ solution, then the gamma we
find is $\gamma = \exp{( \nicefrac{-1}{r})}$. Of course, if we wish to find a
solution with full-rank, then the bounds we stated above match the
existing bounds.

In order to establish our results we rely on a notion introduced in
the statistical community called restricted strong convexity. This
assumption has connections to ideas such as the Restricted Isometry
Property, Restricted Eigenvalue Condition, and Incoherence. In the
work by Shalev-Shwartz, Gonen, and Shamir~\cite{ShalevShwartz:2011vi}
they present results under a form of strong convexity condition
imposed over matrices. Under that setting, the authors demonstrate
that
\begin{equation*}
  \ell(\Theta_k) \geq \ell (\Theta^*) - \frac{\ell(0) r}{k}
\end{equation*}
where $r$ is the rank of $\Theta^*$. In contrast, our bound behaves as
\begin{equation*}
  \ell(\Theta_k) \geq \ell(\Theta^*) + (\ell(\Theta^*) - \ell(0)) \exp{(\nicefrac{-k}{r})}
\end{equation*}

\paragraph{Our contributions:} We improve upon the linear rates of convergence
for low-rank approximation using rank one updates by connecting the
coordinate descent problem to that of submodular optimization. We
present this result in the sequel along with the algorithmic
consequences. We demonstrate the good performance of these rank one
updates in the experimental section.
%%% Local Variables:
%%% mode: latex
%%% TeX-master: "lowranksub"
%%% End:

\section{Background}
\label{sec:background}
We begin by fixing some notation. We represent sets using sans script fonts \textit{e.g.} $\sfA, \sfB$. Vectors are represented using lower case bold letters \textit{e.g.} $\bx,\by$, and matrices are represented using upper case bold letters \textit{e.g.} $\bX,\bY$.  Non-bold face letters are used for scalars \textit{e.g.} $j,M,r$ and function names \textit{e.g.} $f(\cdot)$. The transpose of a vector or a matrix is represented by $\top$ \textit{e.g.} $\bX^\top$. Define $[p]:=\{1,2,\ldots, p\}$. For singleton sets, we write $f(j) := f(\{j\})$. Size of a set $\sfS$ is denoted by $|\sfS|$. $\inprod{\cdot}{\cdot}$ is used for matrix inner product. 

Our goal is to analyze greedy algorithms for low rank estimation. Consider the classic greedy algorithm that picks up the next element \emph{myopically} \textit{i.e.} given the solution set built so far, the algorithm picks the next element as the one which maximizes the gain obtained by adding the said element into the solution set. Approximation guarantees for the greedy algorithm readily imply for the subclass of functions called \emph{submodular functions} which we define next. 
\begin{defin}
A set function $f(\cdot) : [p] \rightarrow \bbR$ is submodular if for all $\sfA, \sfB \subseteq [p]$,
\begin{align*}
 f(\sfA) + f(\sfB) \geq f(\sfA \cup \sfB) + f(\sfA \cap \sfB).
 \end{align*}
\end{defin}

Submodular set functions are well studied and have many desirable properties that allow for efficient minimization, and maximization with approximation guarantees. Our low rank estimation problem also falls under the purview of another class of functions called \emph{monotone} functions. A function is called monotone if and only if $f(\sfA) \leq f(\sfB)$ for all $\sfA \subseteq \sfB$. For the specific case of maximizing monotone submodular set functions, it is known that the greedy algorithm run for (say) $k$ iterations is guaranteed to return a solution that is within $(1-\nicefrac{1}{e})$ of the optimum set of size $k$~\citep{nemhauser1978}. Moreover, without further assumptions or knowledge of the function, no other polynomial time algorithm can provide a better approximation guarantee unless P=NP~\citep{feige1998}.

More recently, a line of works have shown that the greedy approximation guarantee that is typically applicable to monotone submodular functions can be extended to a larger class of functions called \emph{weakly} submodular functions~\citep{Elenberg:2016,Khanna:2016vr}. Central to the notion of weak submodularity is the quantity submodularity ratio which we define next. 

\begin{defin}[Submodularity Ratio~\cite{Kempe:2011ue}]\label{def:submodularityRatio}
    Let $\sfS, \sfL \subset [p]$ be two disjoint sets, and $f(\cdot) : [p] \rightarrow \bbR$. The submodularity ratio of $\sfL$ with respect to $\sfS$ is given by
    \begin{align}
    \gamma_{\sfL,\sfS} := \frac{\sum_{j\in \sfS} \left[f(\sfL \cup \{j\}) - f(\sfL) \right]}{f(\sfL \cup \sfS) - f(\sfL)} \label{eq:submodDef} .
    \end{align}
    The submodularity ratio of a set $\sfU$ with respect to an integer $k$ is given by
    \begin{align}
    \gamma_{\sfU,k} := \min_{\substack{\sfL,\sfS :  \sfL \cap \sfS = \emptyset ,\\ \sfL \subseteq \sfU,  |\sfS| \leq k}} \gamma_{\sfL,\sfS} .
    \end{align}
\end{defin}

It is easy to show that $f(\cdot)$ is submodular if and only if $\gamma_{\sfL,\sfS} \geq 1$ for all sets $\sfL$ and $\sfS$. However, as noted by~\citet{Kempe:2011ue,Elenberg:2016}, an approximation guarantee is guaranteed when $0 < \gamma_{\sfL,\sfS}\; \forall \sfL,\sfS $, thereby extending the applicability of the greedy algorithm to a much larger class of functions. The subset of monotone functions which have $\gamma_{\sfL,\sfS} > 0\; \forall \sfL,\sfS $ are called weakly submodular functions in the sense that even though the function is not submodular, it still provides provable bound for greedy selections.

Also vital to our analysis is the notion of restricted strong concavity and smoothness~\citep{Negahban2012journal, loh2015}.

\begin{defin}[Low Rank Restricted Strong Concavity, Restricted Smoothness] \label{def:RSCRSM}
    A function $\ell : \bbR^{n\times d} \rightarrow \bbR$ is said to be restricted strong concave with parameter $m_\Omega$ and restricted smooth with parameter $M_\Omega$ if for all $\bX,\bY \in \Omega \subset \bbR^{n \times d}$,
    \begin{align*}
    - \frac{m_\Omega}{2} \pnorm{\bY -\bX}{F}^2 &\geq \ell(\bY) - \ell(\bX) - \langle \nabla \ell(\bX) , \bY - \bX \rangle \\
    & \geq  - \frac{M_\Omega}{2} \pnorm{\bY - \bX}{F}^2 .
     \end{align*}
\end{defin}

\begin{rem}
    If a function $\ell(\cdot)$ has restricted strong concavity parameter $m$, then its negative $-\ell(\cdot)$ has restricted strong convexity parameter $m$. We choose to use the nomenclature of concavity for ease of exposition in terms of relationship to submodular maximization. Further, note that we define RSC/RSM conditions on the space of matrices rather than vectors.
\end{rem}

It is straightforward to see that if $\Omega^{\prime} \subseteq \Omega$, then $M_{\Omega^{\prime}} \leq M_{\Omega}$ and $m_{\Omega^{\prime}} \geq m_{\Omega}$. 

%With slight abuse of notation, let $(m_k,M_k)$ denote the RSC and RSM parameters on the domain of all $k$-sparse vectors. If $j \leq k$, then $M_j \leq M_k$ and $m_j \geq m_k$. In addition, denote $\tilde{\Omega} := \{(\vecx,\vecy): \pnorm{\vecx-\vecy}{0} \leq 1\}$ with corresponding smoothness parameter $\tilde{M}_1 \leq M_1$.

\section{Setup}
\label{sec:setup}
In this section, we delineate our setup of low rank estimation. For the sake of convenience of relating to the framework of weak submodular maximization, we operate in the setting of maximization of a concave matrix variate function under a low rank constraint. This is equivalent to minimizing a convex matrix variate function under the low rank constraint as considered by~\citet{ShalevShwartz:2011vi} or under nuclear norm constraint or regularization as considered by~\citet{Jaggi:2010nn}. The goal is to maximize a function $l: \bbR^{n\times d}\rightarrow \bbR $ under a low rank constraint: 

\begin{equation}
\label{eq:matrixLowRank}
\max_{\text{rank}(\bX) \leq r } \ell (\bX).
\end{equation}

Instead of using a convex relaxation of the constrained problem~\eqref{eq:matrixLowRank}, our approach is to enforce the rank constraint directly by adding rank $1$ matrices greedily until $\bX$ is of rank $k$. The rank $1$ matrices to be added are obtained as outer product of vectors from the given vector sets $\cU$ and $\cV$. \emph{e.g.} $\cU:=\{\bx \in \bbR^n  \,s.t.\, \| \bx \|_2 = 1 \}$ and $\cV:= \{\bx \in \bbR^d  \,s.t.\, \| \bx \|_2 = 1 \} $. 

The problem~\eqref{eq:matrixLowRank} can be interpreted in a context of sparsity as long as $\cU$ and $\cV$ are enumerable. For example, by using the SVD theorem, it is known that we can rewrite $\bX $ as $\sum_{i=1}^k \alpha_i \bu_i \bv_i^\top$, where $\forall i$, $\bu_i \in \cU$ and $\bv_i \in \cV$. By using enumeration of the sets $\cU$ and $\cV$ under a finite precision representation of real values, one can rethink of the optimization~\eqref{eq:matrixLowRank} as finding a sparse solution for the infinite dimensional vector $\balpha$~\citep{ShalevShwartz:2011vi}. As a first step, we can define an optimization over \emph{specified} support sets, similar to choosing support for classical sparsity in vectors. For a support set $\sfL$, let $\bU_\sfL$ and $\bV_\sfL$ be the matrices formed by stacking the chosen elements of $\cU$ and $\cV$ respectively. For the support $\sfL$, we can define a set function that maximizes $\ell(\cdot)$ over $\sfL$:

\begin{equation}
\label{eq:fSdefn}
f(\sfL) = \max_{\bH \in \bbR^{|\sfL| \times |\sfL|}} \ell (\bU_\sfL^\top \bH \bV_\sfL) - \ell (\mathbf{0}).
\end{equation} 

We will denote the optimizing matrix for a support set $\sfL$ as $\betal$. In other words, let $\hat{\bH}_\sfL$ be the argmax obtained in~\eqref{eq:fSdefn}, then $\betal \defn \bU^{\top}_{\sfL} \hat{\bH}_\sfL \bV_\sfL$.

 Thus, the low rank matrix estimation problem~\eqref{eq:matrixLowRank} can be reinterpreted as the following equivalent combinatorial optimization problem: 

\begin{equation}
\label{eq:combinatorialLowRank}
\max_{|\sfS| \leq k } f(\sfS).
\end{equation}

\subsection{Algorithms}
\label{sec:algo}
We briefly state the algorithms. Our greedy algorithm is illustrated in Algorithm~\ref{alg:greedy}. The greedy algorithm builds the support set incrementally -- adding a rank $1$ matrix one at a time, so that at iteration $i$ for $1 \leq i \leq k$ the size of the chosen support set and hence rank of the current iterate is $i$. We assume access to a subroutine $\greedysel$ for the greedy selection (Step 4). This subroutine solves an inner optimization problem by calling a subroutine $\greedysel$ which returns an atom $s$ from the candidate support set that ensures
\begin{align*}
f(\sfS_{i-1}^G \cup \{s\}) - f(\sfS_{i-1}^G) \geq  \tau \left( f(\sfS_{i-1}^G \cup \{s^\star\}) - f(\sfS_{i-1}^G) \right),
\end{align*}

where 
\begin{align*}
s^\star \gets \argmax_{ a \in (\cU \times \cV) \perp \sfS_{i-1}^G}  f(\sfS_{i-1}^G \cup \{a\}) - f(\sfS_{i-1}^G).
\end{align*}

In words, the subroutine $\greedysel$ ensures that the gain in $f(\cdot)$ obtained by using the selected atom is within $\tau \in (0,1]$ multiplicative approximation to the atom with the best possible gain in  $f(\cdot)$. The hyperparameter $\tau$ governs a tradeoff allowing a compromise in myopic gain for a possibly quicker selection. 

The greedy selection requires to fit and score every candidate support, which is prohibitively expensive. An alternative is to choose the next atom by using the linear maximization oracle used by Frank-Wolfe~\citep{jaggi13FW} or Matching Pursuit algorithms~\citep{Gribonval:2006ch}. This step replaces Step 4 of Algorithm~\ref{alg:greedy} as illustrated in Algorithm~\ref{alg:omp}. Let $\sfL = \sfS^O_{i-1}$ be the set constructed by the algorithm at iteration $(i-1)$. The linear oracle $\OMPsel$ returns an atom $s$ for iteration $i$ ensuring

\begin{align*}
\inprod{\nabla \ell (\betal )}{ \bu_s \bv_s^\top} \geq \tau \max_{(\bu, \bv) \in  (\cU \times \cV) \perp \sfS_{i-1}^O} \inprod{\nabla \ell (\betal )}{ \bu \bv^\top} .
\end{align*}

The linear problem $\OMPsel$ can be considerably faster that $\greedysel$. $\OMPsel$\, reduces to finding the left and right singular vectors of $\nabla \ell(\betal)$ corresponding to its largest singular value, which is $O(\frac{t}{1 - \tau} (\log n + \log d))$, where $t$ is the number of non-zero entries of $\nabla \ell(\betal)$. 

Algorithm~\ref{alg:omp} is the same as considered by~\citet{ShalevShwartz:2011vi} as GECO (Greedy Efficient Component Optimization). However, as we shall see, our analysis provides stronger bounds than their Theorem 2.

\begin{algorithm}[H]
    \caption{\func($\cU$, $\cV$, $k$, $\tau$)}
   \label{alg:greedy}
\begin{algorithmic}[1]
	\STATE {\bfseries Input:} sparsity parameter $k$, vector sets $\cU$, $\cV$
	\STATE $\sfS_0^G \gets \emptyset$
	\FOR{$i=1\ldots k$}
	\STATE $s \gets \greedysel(\tau)$
	\STATE $\sfS_i^G \gets \sfS_{i-1}^G \cup \{s\}$
	\ENDFOR
	\STATE {\bfseries return} $\sfS_k^G$, $\bB^{(\sfS_k^G)}$, $f(\sfS_k^G)$.
\end{algorithmic}
\end{algorithm}

\begin{algorithm}[h]
    \caption{\geco($\cU$, $\cV$, $k$, $\tau$)}
   \label{alg:omp}
\begin{algorithmic}
%\SetAlgoNoLine
\def\NoNumber#1{{\def\alglinenumber##1{}\STATE #1}}
	\STATE { same as Algorithm~\ref{alg:greedy} except}
	\STATE \hspace{-0.5em}{\small4:} $\,s \gets \OMPsel(\tau)$
%	\STATE 4: $s \gets \OMPsel(\tau)$
\end{algorithmic}
\end{algorithm}

\begin{rem}
We note that Step 5 of Algorithms~\ref{alg:greedy},\ref{alg:omp} requires solving the RHS of~\eqref{eq:fSdefn} which is a matrix variate problem of size $i^2$ at iteration $i$. This refitting is equivalent to the ``fully-corrective" versions of Frank-Wolfe/Matching Pursuit algorithms which, intuitively speaking, extract out all the information w.r.t $\ell(\cdot)$ from the chosen set of atoms, thereby ensuring that the next rank $1$ atom chosen has row and column space orthogonal to the previously chosen atoms. Thus the constrained maximization on the orthogonal complement of $\sfS^{G}_i$ in subroutines $\OMPsel$ and $\greedysel$ need not be explicitly enforced, but is still shown for clarity.
\end{rem}

\section{Analysis}
\label{sec:analysis}

In this section, we prove that low rank matrix optimization over
the rank one atoms satisfies weak submodularity. This helps us bound the function value  obtained till $k$ greedy iterations vis-a-vis the function value at the optimal $k$ sized selection. 

We explicitly delineate some notation and assumptions. With slight abuse of notation, we assume $\ell(\cdot)$ is $m_i$-strongly concave and $M_i$-smooth over matrices of rank $i$. For $i \leq j$, note that $m_i \geq m_j$ and $M_i \leq M_j$. Additionally, let $\tilde{\Omega} := \{ (\bX,\bY): \text{rank}(\bX- \bY) \leq 1 \}$, and assume $\ell(\cdot)$ is $\Mo$-smooth over $\tilde{\Omega}$. It is easy to see $\Mo \leq M_1$.

Since we obtain approximation bounds for the greedy algorithm similar to ones obtained using classical methods, we must also mention the corresponding assumptions in the submodular literature that further draws parallels to our analysis. Submodularity guarantees that greedy maximization of \emph{monotone} \emph{normalized} functions yields a $(1-\nicefrac{1}{e})$ approximation. Since we are doing support selection, increasing the support size does not decrease the function value. Hence the set function we consider is monotone. Further, we also subtract $\ell(\mathbf{0})$ to make sure $f(\emptyset)=0$ (see~\eqref{eq:fSdefn}) so that our set function is also normalized. We shall see that our bounds are of similar flavor as the classical submodularity bound.

As the first step, we prove that if the
low rank RSC holds, then the submodularity ratio (Definition~\ref{def:submodularityRatio}) is lower-bounded by
the inverse condition number. 

\begin{theorem}
\label{thm:rscImpliesSubmod}
  Let $\sfL$ be a set of $k$ rank $1$ atoms and $\sfS$
be a set of $r$ rank $1$ atoms where we sequentially orthogonalize the
atoms against $\sfL$. If  $\,\ell(\cdot)$ is $m_i$-strongly concave over matrices of rank $i$, and $\Mo$-smooth over the set $\tilde{\Omega}:=\{(\bX,\bY): \text{rank}(\bX-\bY)=1\}$, then
  \begin{equation*}
    \gamma_{\sfL,r} \defn \frac{\sum_{a \in S} [f(\sfL \cup \{a\}) - f(\sfL)]}{f(\sfL \cup \sfS) - f(\sfL)} \geq \frac{m_{r+k}}{\Mo}.
  \end{equation*}
\end{theorem}

%Should $v$ and $\alpha_a$ have projections onto $\sfS \cap C$? $P_{\sfS}^{C}$? Modify the RSC definition?
%Use absolute values ensure that $\alpha_a \in C$ and $\alpha_a \inprod{\betals}{a} \in C$ so that $y_a$ remains feasible. 

The proof of Theorem~\ref{thm:rscImpliesSubmod} is structured around individually obtaining a lower bound for the numerator and an upper bound for the denominator of the submodularity ratio by exploiting the concavity and convexity conditions. Bounding the submodularity ratio is crucial to obtaining the approximation bounds for the greedy algorithm as we shall see in the sequel.

\subsection{Greedy Improvement}
\label{sec:improve}
In this section, we obtain approximation guarantees for Algorithm~\ref{alg:greedy}. 

\begin{theorem}
\label{thm:approxGreedy}
Let $\sfS:=\sfS_k^G$ be the greedy solution set obtained by running Algorithm~\ref{alg:greedy} for $k$ iterations, and let $\sfS^\star$ be an optimal support set of size $r$. Let $\ell(\cdot)$ be $m_i$ strongly concave on the set of matrices with rank less than or equal to $i$, and $\Mo$ smooth on the set of matrices in the set $\tilde{\Omega}$. Then, 
\begin{align*}
f(\sfS) &\geq ( 1 - \frac{1}{e^{ c_1} }) f(\sfS^\star)  \\
&\geq ( 1 - \frac{1}{e^{c_2} }) f(\sfS^\star) ,
\end{align*}
where $c_1 = \tau \gamma_{\sfS,r} \frac{k  }{r} $ and $c_2 = { \tau \frac{m_{r+k}}{\Mo} \frac{k  }{r} }$. 
\end{theorem}

The proof technique for the first inequality of
Theorem~\ref{thm:approxGreedy} relies on lower bounding the progress
made in each iteration of Algorithm~\ref{alg:greedy}. Intuitively, it
exploits the weak submodularity to make sure that each iteration makes
\emph{enough} progress, and then applying an induction argument for
$r$ iterations. We also emphasize that the bounds in
Theorem~\ref{thm:approxGreedy} are for \emph{normalized} set function
$f(\cdot)$ (which means $f(\emptyset) = 0$). A more detailed proof is
presented in the appendix.
% The second inequality in Theorem~\ref{thm:approxGreedy} follows from the first one by applying Theorem~\ref{thm:rscImpliesSubmod}. Thus, Theorem~\ref{thm:rscImpliesSubmod}.

\begin{rem}
Theorem~\ref{thm:approxGreedy} provides the approximation guarantees
for running the greedy selection algorithm up to $k$ iterations to obtain a rank $k$ matrix iterate \emph{vis-a-vis} the best rank $r$ approximation. For $r=k$, and $\tau=1$, we get an approximation bound $(1-{e^{ \nicefrac{-m}{M}}})$ which is reminiscent of the greedy bound of $(1 - \nicefrac{1}{e})$ under the framework of submodularity. Note that our analysis can not be used to establish classical submodularity. However, establishing weak submodularity that lower bounds $\gamma$ is sufficient to provide slightly weaker than classical submodularity guarantees. 
\end{rem}

\begin{rem}
Theorem~\ref{thm:approxGreedy} implies that to obtain $(1 - \epsilon)$ approximation guarantee in the worst case, running Algorithm~\ref{alg:greedy} for $k = \frac{rM}{m \tau} \log \frac{1}{\epsilon}) = O(r \log \nicefrac{1}{\epsilon})$ iterations suffices. This is useful when the application allows a tradeoff: compromising on the low rank constraint a little to achieve tighter approximation guarantees. 
\end{rem}

\begin{rem}
\citet{Kempe:2011ue} considered the special case of greedily maximizing $R^2$ statistic for linear regression, which corresponds to classical sparsity in vectors. They also obtain a bound of $(1 - \nicefrac{1}{e^\gamma})$, where $\gamma$ is the submodularity ratio for their respective setup. This was generalized by~\citet{Elenberg:2016} to general concave functions under sparsity constraints. Our analysis is for the low rank constraint, as opposed to sparsity in vectors that was considered by them.
\end{rem}

\subsection{GECO Improvement}
In this section, we obtain the approximation guarantees for Algorithm~\ref{alg:omp}. The greedy search over the infinitely many candidate atoms is infeasible, especially when $\tau=1$. Thus while Algorithm~\ref{alg:greedy} establishes interesting theoretical connections with submodularity, it is, in general, not practical. To obtain a tractable and practically useful algorithm, the greedy search is replaced by a Frank Wolfe or Matching Pursuit style linear optimization which can be easily implemented as finding the top singular vectors of the gradient at iteration $i$. In this section, we show that despite the speedup, we lose very little in terms of approximation guarantees. In fact, if the approximation factor $\tau$ in  \OMPsel() is $1$, we get the same bounds as those obtained for the greedy algorithm. 

We now present our main result for Algorithm~\ref{alg:omp}.

\begin{theorem}
\label{thm:approxOMP}
Let $\sfS:= \sfS_k^O$ be the greedy solution set obtained using Algorithm~\ref{alg:omp} for $k$ iterations, and let $\sfS^\star$ be the optimum size $r$ support set. Let $\ell(\cdot)$ be $m_{r+k}$ strongly concave on the set of matrices with rank less than or equal to $(r+k)$, and $\Mo$ smooth on the set of matrices with rank in the set $\tilde{\Omega}$. Then, 
\begin{align*}
f(\sfS) & \geq ( 1 - \frac{1}{e^{ c_3} }) f(\sfS^\star) ,
\end{align*}

where $c_3 = \tau^2 \frac{m_{r+k}}{\Mo} \frac{k  }{r} $.
\end{theorem}

The proof of Theorem~\ref{thm:approxOMP} follows along the lines of Theorem~\ref{thm:approxGreedy}. The central idea is similar - to exploit the RSC conditions to make sure that each iteration makes \emph{sufficient} progress, and then provide an induction argument for $r$ iterations. Unlike the greedy algorithm, however, using the weak submodularity is no longer required. Note that the bound obtained in Theorem~\ref{thm:approxOMP} is similar to Theorem~\ref{thm:approxGreedy}, except the exponent on the approximation factor $\tau$. 

\begin{rem}
Our proof technique for Theorem~\ref{thm:approxOMP} can be applied for classical sparsity to improve the bounds obtained by~\citet{Elenberg:2016} for OMP for support selection under RSC, and by~\citet{Kempe:2011ue} for $R^2$ statistic. If $\tau=1, r = k$, their bounds involve terms of the form $O(\nicefrac{m^2}{M^2})$ in the exponent, as opposed to our bounds which only has $\nicefrac{m}{M}$ in the exponent.
\end{rem}

%\begin{rem}
%Similar to the greedy algorithm, to achieve a tighter approximation to best rank $k$ solution, one can relax the low rank constraint a little by running the algorithm for $r > k$ greedy iterations. The result obtained by our Theorem~\ref{thm:approxOMP} can be compared to the bound obtained by~\cite{ShalevShwartz:2011vi} [Theorem 2] for the same algorithm. For an $\epsilon$ multiplicative approximation, Theorem~\ref{thm:approxOMP} implies we need $\nicefrac{r}{k} = O(\log \nicefrac{1}{\epsilon}) $. On the other hand,~\cite{ShalevShwartz:2011vi} obtain an additive approximation bound with $\nicefrac{r}{k}= O(\nicefrac{1}{\varepsilon})$, which is an exponential improvement. 
%\end{rem}
%%
%\begin{rem}
%We must also note that our RSC assumptions are much weaker than those made by~\citet{ShalevShwartz:2011vi} for their Theorem 2. They assume RSC condition directly on the infinite dimensional vector $\balpha$ over a specified support set as specified in Section~\ref{sec:setup}. However, for any fixed support set, it is easy to show that their RSC condition is too strong to hold as specified (see Appendix~\ref{sec:appendix:rscfail} for details). On the other hand, our RSC is specified over the space of low rank matrices which is much weaker and a more reasonable assumption to make. We must note that this does not discredit other results presented by~\citet{ShalevShwartz:2011vi} which do not need the strong convexity assumptions. 
%\end{rem}
\section{Recovery Guarantees}

While understanding approximation guarantees are useful, providing parameter recovery bounds can further help us understand the practical utility of the greedy algorithm. In this section, we present a general theorem that provides us with recovery bounds of the true underlying low rank structure. 

\begin{theorem}
\label{thm:recovery}
Suppose that an algorithm achieves the approximation guarantee:

\begin{equation*}
f(\sfS_k) \geq C_{r,k} f(\sfS^\star_r),
\end{equation*}

where $\sfS_k$ is the set of size $k$ at iteration $k$ of the algorithm, $\sfS^\star_r$ be the optimal solution for $r$-cardinality constrained maximization of $f(\cdot)$, and $C_{r,k}$ be the corresponding approximation ratio guaranteed by the algorithm. Recall that we represent by $\bU_\sfS,\bV_\sfS$ the matrices formed by stacking the vectors represented by the support set $\sfS$ chosen from $\cU,\cV$  respectively, s.t. $| \sfS| =r$. Then under $m_{k+r}$ RSC,  with $\bB_\sfr = \bU_\sfS^\top \bH \bV_\sfS$ for any $\bH \in \bbR^{r \times r} $, we have 

\begin{align*}
\| \bB^{(\sfS_k)} -  \bB_\sfr \|_F^2  & \leq    4(k+r)  \frac{ \| \nabla \ell (\bB_\sfr)  \|^2_2}{m_{k+r}^2} \\
&+  \frac{4(1 - C_{r,k} )}{m_{k+r}} [  \ell (\bB_\sfr)- \ell(\mathbf{0})]
\end{align*}

\end{theorem}

Theorem~\ref{thm:recovery} can be applied for $\bB_\sfr  = \bB^{(\sfS_r^\star)}$, which is the argmax for maximizing $\ell(\cdot)$ under the low rank constraint. It is general - in the sense that it can be applied for getting recovery bounds from approximation guarantees for any algorithm, and hence is applicable for both Algorithms~\ref{alg:greedy} and~\ref{alg:omp}. 

For specific function $\ell(\cdot)$ and statistical model, statistical recovery guarantees guarantees can be obtained from Theorem~\ref{thm:recovery} for specific $\ell(\cdot)$ and statistical model, Consider the case of low rank matrix estimation from linear measurements. Say $\bX_i \in \bbR^{m_1 \times m_2}$ for $i \in [n]$ are generated so that each entry of $\bX_i$ is $\cN(0,1)$. We observe $\by_i = \inprod{\bX_i}{\Theta^\star} + \varepsilon$, where $\Theta^\star$ is low rank, and say $\varepsilon\sim \cN(0,\sigma^2)$. Let $N = m_1m_2$, and let $\varphi (\Theta): \bbR^{m_1 \times m_2} \rightarrow \bbR^n$ be the linear operator so that $[\varphi (\Theta)]_i =\inprod{\bX_i}{\Theta} $.  Our corresponding function is now $\ell(\Theta) = -\frac{1}{n} \| \by -\varphi (\Theta) \|^2_2 $. For this function, using arguments by ~\citet{Negahban2012journal}, we know $\|\nabla \ell(\bB^{\sfS^\star_r}) \|_2^2 \leq \frac{\log N}{n}$ and $\ell(\bB^{\sfS^\star_r}) - \ell(\mathbf{0)} \leq (s+1)$ with high probability. It is also straightforward to apply their results to bound $m_{k+r} \geq \left( \frac{1}{32} - \frac{162(k+r) \log N}{n} \right)$, and  $M_1\leq 1$, which gives explicit bounds as per Theorem~\ref{thm:recovery} for Algorithms~\ref{alg:greedy},~\ref{alg:omp} for the considered function and the design matrix.

\section{Experiments}
\label{sec:experiments}
In this section, we empirically evaluate the proposed algorithms. 

\subsection{Clustering under Stochastic Block Model}
\label{sec:clusteringexpts}
In this section, we test empirically the performance of GECO (Algorithm~\ref{alg:omp}) for a clustering task. We are provided with a graph with nodes and the respective edges between the nodes. The observed graph is assumed to have been noisily generated from a true underlying clustering. The goal is to recover the underlying clustering structure from the noisy graph provided to us. The adjacency matrix of the true underlying graph is low rank. As such, our greedy framework is applicable . We compare performance of Algorithm~\ref{alg:omp} on simulated data against standard baselines of spectral clustering which are commonly used for this task. We begin by describing a generative model for creating edges between nodes given the ground truth.

The Stochastic Block Model is a model to generate random graphs. It takes its input the set of $n$ nodes, and a partition of $[n]$ which form a set of disjoint clusters, and returns the graph with nodes and the generated edges. The model is also provided with generative probabilities $(p,q)$ -- so that a pair of nodes within the same cluster have an edge between them with probability $p$, while a pair of nodes belonging to different clusters have an edge between them with probability $q$. For simplicity we assume $q=(1-p)$. The model then iterates over each pair of nodes. For each such pair that belongs to same cluster, it samples an edge as Bernoulli($p$), otherwise as Bernoulli($1-p$). This provides us with a $\{0,1\}$ adjacency matrix.

For baselines, we compare against two versions of spectral clustering, which is a standard technique applied to find communities in a graph. The method takes as input the $n \times n$ adjacency matrix $\bA$, which is a $\{0,1\}$ matrix with an entry $\bA_{ij}= 1 $ if there is an edge between node $i$ and $j$, and is $0$ otherwise. From the adjacency matrix, the graph Laplacian $\bL$ is constructed. The Laplacian may be unnormalized, in which case it is simply $\bL = \bD - \bA$, where $\bD$ is the diagonal matrix of degrees of nodes. A normalized Laplacian is computed as $\bL_{\text{norm}} = \bD^{-\nicefrac{1}{2}} \bL  \bD^{-\nicefrac{1}{2}}$. After calculating the Laplacian, the algorithm solves for bottom $k$ eigenvectors of the Laplacian, and then apply $k$-means clustering on the rows of the thus obtained eigenvector matrix. We refer to the works of~\citet{Shi:2000nc,Ng:2001spectral} for the specific details of clustering algorithms using unnormalized and normalized graph Laplacian respectively.

We compare the spectral clustering algorithms with logistic PCA, which is a special case of the exponential family PCA~\citep{Collins:2001pca}. The exponential family extension of classical PCA is analogous to the extension of the linear regression to generalized linear models (GLMs). For a given matrix $\bX$, the GLM generative model assumes that each cell $\bX_{ij}$ is independently drawn with likelihood proportional to $\exp{\inprod{\Theta_{ij}}{\bX_{ij}} - G(\Theta_{ij})}  $, where $\Theta$ is the true underlying parameter, and $G(\cdot)$ is the partition function. It is easy to see we can apply our framework of greedy selection by defining $\ell(\cdot)$ as the log-likelihood:
\begin{equation*}
\ell(\Theta) =  \inprod{\Theta}{\bX} - \sum_{i,j} G(\Theta_{ij}),
\end{equation*}

where $\Theta$ is the true parameter matrix of $p$ and $q$ that generates a realization of $\bA$. Since the true $\Theta$ is low rank, we get the low rank constrained optimization problem: 

\begin{equation*}
\max_{\text{rank}(\Theta) \leq k} \ell(\Theta),
\end{equation*}

where $k$ is the hyperparameter that is suggestive of true number of clusters. Note that  lack of knowledge of true value of $k$ is not more restrictive than spectral clustering algorithms which typically also require the true value of $k$, albeit some subsequent works have tried to address tuning for $k$. 

Having cast the clustering problem  in the same form as~\eqref{eq:matrixLowRank}, we can apply our greedy selection algorithm as opposed to the more costly alternating minimizing algorithms suggested by~\citet{Collins:2001pca}. Since the given matrix is $\{0,1\}$ with each entry sampled from a Bernoulli, we use $G(x) = \log (1+ e^x)$ which gives us logistic PCA. 

We generate the data as follows. For $n=100$ nodes, and fixed number of cluster $k=5$, we vary the within cluster edge generation probability $p$ from $0.55$ to $0.95$ in increments of $0.05$, and use the Stochastic Block model to generate a noisy graph with each $p$. Note that smaller $p$ means that the sampled graph will be more noisy and likely to be more different than the underlying clustering.

 We compare against the spectral clustering algorithm using unnormalized Laplacian of~\citet{Shi:2000nc} which we label ``Spectral\_unnorm\{k\}" for $k=\{3,5,10\}$, and the spectral clustering algorithm using normalized Laplacian of~\citet{Ng:2001spectral} which we label ``Spectral\_norm\{k\}" for $k=\{3,5,10\}$. We use Algorithm~\ref{alg:omp} which we label ``Greedy\{k\}" for $k=\{3,5,10\}$. For each of these models, the referred $k$ is the supplied hyperparameter. We report the least squares error of the output from each model to the true underlying $\Theta$ (which we call the generalization error), and to the instantiation used for training $\bX$ (which we call the reconstruction error). The results are presented in Figure~\ref{fig:clustering}.
 
 Figure~\ref{fig:clustering} shows that the greedy logistic PCA performs well in not only re-creating the given noisy matrix (reconstruction) but also captures the true low rank structure better (generalization). Further, note that providing the true hyper parameter $k$ is vital for spectral clustering algorithms, while on the other hand greedy is less sensitive to $k$ which is very useful in practice as $k$ is typically not known. So the spectral clustering algorithms typically would involve taking an SVD and re-running the $k-means$ for different values of $k$ to choose the best performing hyperparameter. The greedy factorization on the other hand is more robust, and moreover is incremental - it does not require to be re-run from scratch for different values of $k$.

\begin{figure}[t]
\subfigure{\includegraphics[width=71mm,height=56mm]{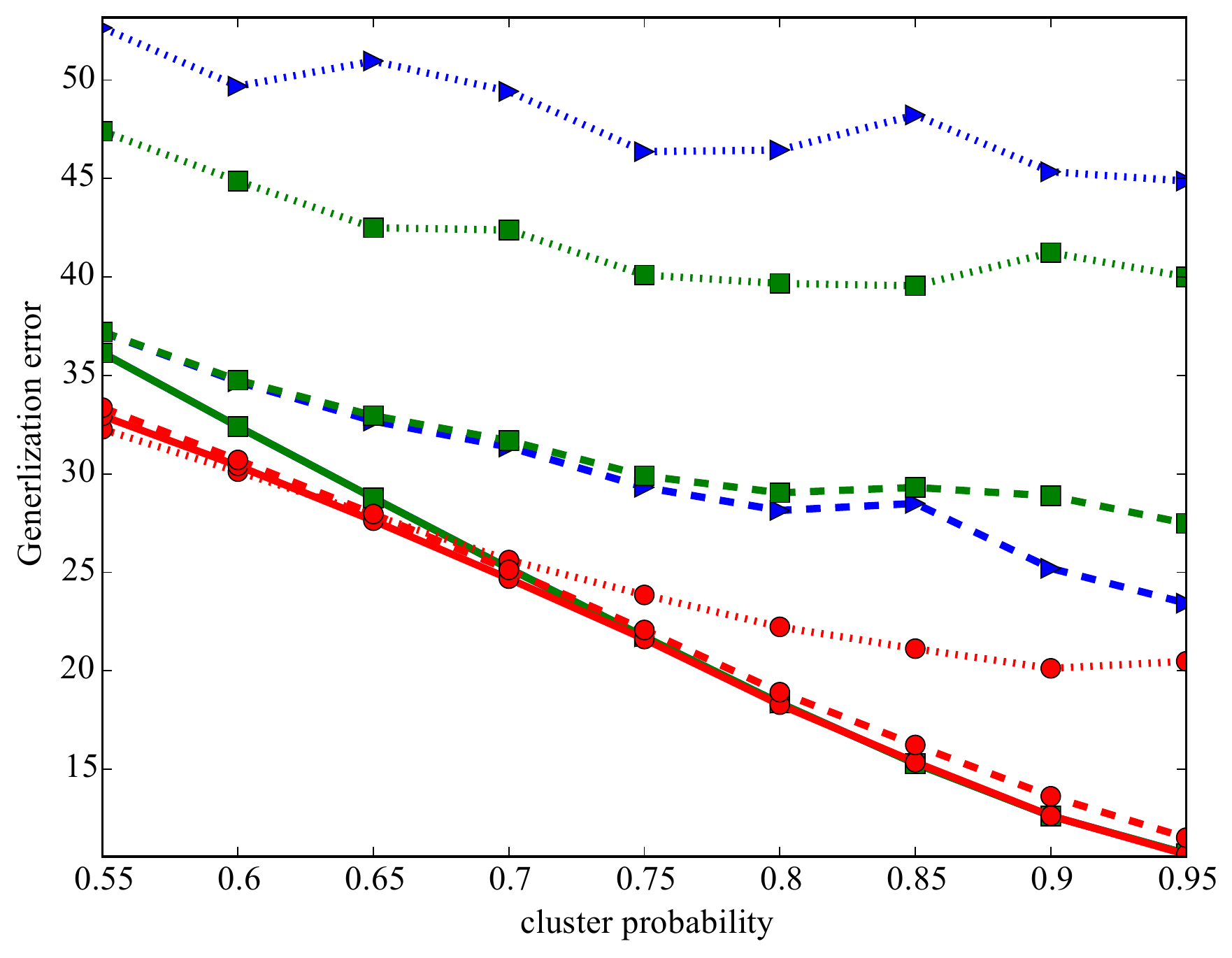}}
\subfigure{\includegraphics[width=95mm, height=56mm]{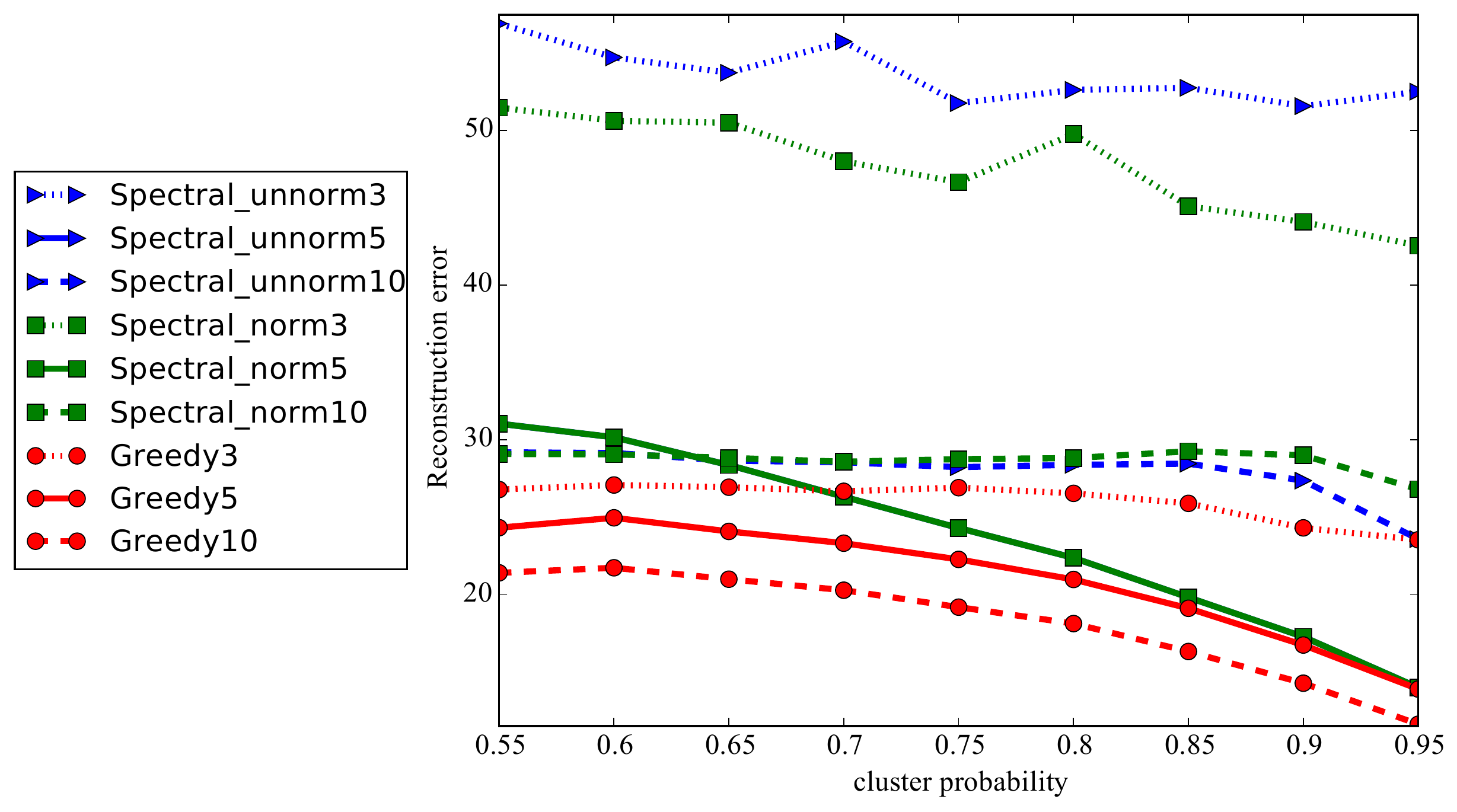}}
\caption{Greedy Logistic PCA vs spectral clustering baselines averaged over 10 runs. \emph{Top}: Robust performance of the greedy logistic PCA for generalizing over varying values of $k$ across different values of $p$, spectral clustering algorithms are more sensitive to knowing true value of $k$ \emph{Bottom}: Strong performance of greedy logisitic PCA even with small value of $k=3$ for reconstructing the given cluster matrix.}
% \centering
%\begin{subfigure}{0.5\textwidth}
% \centering
%\includegraphics[scale=0.5]{figures/graphon_Generlization}
%\caption{Lorem ipsum}
%\end{subfigure}
%~
%\begin{subfigure}{0.5\textwidth}
% \centering
%\includegraphics[scale=0.5]{figures/graphon_Reconstruction}
%\caption{Lorem ipsum}
%\end{subfigure}
%\caption{BIG ipsum}
\label{fig:clustering}
\end{figure}

\subsection{Word Embeddings}
The task of embedding text into a vector space yields a representation that can have many advantages, such as using them as features for subsequent tasks as sentiment analysis.~\citet{Mikolov:2013dr} proposed a context-based embedding called skip-gram or \texttt{word2vec}). The context of a word can be defined as a set of words before, around, or after the respective word. Their model strives to find an embedding of each word so that the representation predicts the embedding of each context word around it.  In a recent paper,~\citet{Omer:2014embedding} showed that the word embedding model proposed by~\citet{Mikolov:2013dr} can be re-interpreted as matrix factorization of the \emph{PMI} matrix constructed as follows. A word $c$ is in context of $w$ if it lies within the respective window of $w$. The PMI matrix is then calculated as 
\newcommand{\pmi}{\text{PMI}}
\newcommand{\ppmi}{\text{PPMI}}

\begin{equation*}
\pmi_{w,c} = \log \left( \frac{p(w,c)}{p(w)p(c)} \right).
\end{equation*}

In practice the probabilities $p(w,c), p(w), p(c)$ are replaced by their empirical counterparts. Further, note that $p(w,c)$ is $0$ if words $c$ and $w$ do not co-exist in the same, context which yields $-\infty$ for PMI. ~\citet{Omer:2014embedding} suggest using an alternative: $\ppmi_{w,c} = \max\{\pmi_{w,c}, 0\}$. They also suggest variations of PMI hyper parameterized by $k$ which corresponds to the number of negative samples in the training of skip gram model of~\citet{Mikolov:2013dr}. 

We employ the binomial model on the normalized count matrix (instead of the PMI), in a manner similar to the clustering approach in Section~\ref{sec:clusteringexpts}. The normalized counts matrix is calculated simply as $\frac{p(w,c)}{p(w)} $, without taking explicit logs unlike the \pmi ~ matrix. This gives us a probability matrix which has each entry between $0$ and $1$, which can be factorized under the binomial model greedily as per Algorithm~\ref{alg:omp}, similar to the way we do it in Section~\ref{sec:clusteringexpts}.

We note that embeddings using the SVD is more scalable than our greedy approach because of advancements in linear algebraic techniques for SVD on sparse matrices that PPMI yields.  Our experiments show that binomial PCA can be competitive to other existing embedding methods. Since our current implementation is not as scalable, we are further investigating this as on-going work.

We empirically study the embeddings obtained by binomial factorization on two tasks - word similarity and analogies. For word similarity, we use the W353 dataset~\citep{Finkelstein:2001:PSC} which has 353 queries and the MEN data~\citep{Bruni:2012:DST} which has 3000 queries. Both these datasets contain words with human assigned similarity scores. We evaluate the embeddings by their cosine similarity, and measuring the correlation with the available human ratings. For the analogy task, we use the Microsoft Research (MSR) syntactic analogies~\citep{Mikolov:2013lr} which has 8000 queries, and the Google mixed analogies dataset~\citep{mikolov:2013ee} with 19544 queries. To compute accuracy, we use the multiplication similarity metric as used by~\citet{Omer:2014embedding}. To train the word embeddings, we use the 2013 news crawl dataset\thanks{http://www.statmt.org/wmt14/training-monolingual-news-crawl}. We filter out stop words and non-ascii characters, and keep only the words which occurs atleast 2000 times which yields vocabulary of 6713. Note that since we filter only the most common words, several queries from the datasets are invalid because we do not have embeddings for words appearing in them. However, we do include them and report the overall average over the entire dataset, with metric being $0$ by default for each query we are not able to process. 

Table~\ref{tab:embeddings} shows the empirical evaluation. SVD and PPMI are the models proposed by~\citet{Omer:2014embedding}, while SGNS is skipgram with negative sampling model of~\citet{Mikolov:2013dr}.  We run each of these for $k=\{5,10,15,20\}$ and report the best numbers. The results show that alternative factorizations such as our application of binomial PCA to those of taking SVD of the \ppmi ~ matrix can be more consistent and competitive with other embedding methods.

\begin{table}
\centering
\begin{tabular}{ c | c | c| c | c  }
  & W353 & MEN & MSR & Google \\ \hline
SVD  & 0.226  & 0.233 & 0.086 & 0.092  \\
PPMI  & 0.175 & 0.178 & 0.210 & 0.130 \\
SGNS & 0.223 & 0.020 & 0.052 & 0.002 \\
Greedy & 0.202 & 0.198 & 0.176 & 0.102
\end{tabular}
\caption{Empirical study of binomial based greedy factorization shows competitive performance of word embeddings of common words across tasks and datasets.}
\label{tab:embeddings}
\end{table}

\section{Conclusion}
We have connected the problem of greedy low-rank matrix estimation to that of submodular optimization. Through that connection we have provided improved exponential rates of convergence for the algorithm. An interesting area of future study will be to connect these ideas to general atoms or dictionary elements.

%\section{TODOs/Reminders}
%\begin{itemize}
%\item greedy matrix pursuit : cite papers by rong ge, arora.
%\end{itemize}

%\bibliographystyle{icml2017}
%\bibliographystyle{ieeetran}
\bibliographystyle{IEEEtranN}

\bibliography{biblio,sahandbib}
\newpage
{
\appendix
\section{Supplement}
In this section, we provide the missing proofs. 

\subsection{Proof of Theorem~\ref{thm:rscImpliesSubmod}}
\begin{proof}
An important aspect of the assumptions is that the space of atoms spanned by $\sfS$ is
orthogonal to the span of $\sfL$. Furthermore,
$\Sspan(\sfL \cup \sfS) \supset \Sspan(\sfS)$. Let $\bar{k} = k+r$. We will first upper bound the denominator in the submodularity ratio. From strong concavity, 
\begin{align*}
  \frac{\mkb}{2} \|\betals - \betal\|_F^2 & \leq \ell(\betal) - \ell(\betals) + \inprod{\nabla \ell(\betal)}{\betals - \betal}
\end{align*}
Rearranging 
\begin{align*}
  0 \leq \ell(\betals) - \ell(\betal) & \leq \inprod{\nabla \ell(\betal)}{\betals-\betal} - \frac{\mkb}{2}
                                  \|\betals - \betal\|_F^2 \\
                                & \leq \argmax_{\substack{\bX: \\ \bX = \bU_{\sfL \cup \sfS} \bH \bV_{\sfL \cup \sfS} \\ \bH \in \bbR^{|\sfL \cup \sfS| \times |\sfL \cup \sfS|}}} \inprod{\nabla
                                  \ell(\betal)}{\bX - \betal}
                                  - \frac{\mkb}{2} \| \bX - \betal\|_F^2 \\
                                & = \argmax_{\substack{\bX: \\ \bX = \bU_{\sfL \cup \sfS} \bH \bV_{\sfL \cup \sfS} \\ \bH \in \bbR^{|\sfL \cup \sfS| \times |\sfL \cup \sfS|}}}
                                  \inprod{P_{\bU_\sfS}(\nabla \ell(\betal)) P_{\bV_\sfS} }{\bX - \betal}
                                  -
                                  \frac{\mkb}{2} \| \bX - \betal\|_F^2, 
\end{align*}

where the last equality holds because $\inprod{ (\nabla \ell(\betal))}{ P_{\bU_\sfL}\bX P_{\bV_\sfL} - \betal} = 0$. Solving the argmax problem, we get $\bX= \betal + \frac{1}{\mkb}P_{\bU_\sfS}( \nabla \ell(\betal)) P_{\bV_\sfS}$. Plugging in,  we get, 

\begin{equation*}
%  \ell(\betals) - \ell(\betal) \leq \frac{\mkb}{2} \|P_{\sfS}(\frac{1}{\mkb} \nabla \ell(\betal))\|_2^2
   \ell(\betals) - \ell(\betal) \leq \frac{1}{2\mkb} \|P_{\bU_\sfS}(\nabla \ell(\betal))  P_{\bV_\sfS} \|_F^2
\end{equation*}

We next bound the numerator. Recall that the atoms in $\sfS$ are orthogonal to each other \textit{i.e.} $\bU_\sfS$ and $\bV_\sfS$ are both orthonormal.

For clarity, we define the shorthand, $\betals_{ij} =  \inprod{\bu_i\bv_j^\top} {\betals} \,  \bu_i \bv_j^\top$, for $i,j \in [| \sfL \cup \sfS| ]. $

With an arbitrary $i \in \sfS$, and arbitrary scalars $\alpha_{ii}, \alpha_{ij}, \alpha_{ji} $ for $j \in \sfL$,

\begin{align*}
  \ell(\bB^{(\sfL \cup \{i\})}) - \ell(\betal) & \geq  \ell(\betal + \alpha_{ii} \betals_{ii} + \sum_{j \in \sfL} \alpha_{ij} \betals_{ij} + \sum_{j \in \sfL} \alpha_{ji} \betals_{ji} ) - \ell(\betal) \\
                                           & \geq  \inprod{\nabla \ell(\betal)}{\alpha_{ii} \betals_{ii} + \sum_{j \in \sfL} \alpha_{ij} \betals_{ij} + \sum_{j \in \sfL} \alpha_{ji} \betals_{ji} }\\
                                           & \;\;\;\; -
                                             \frac{\Mo}{2} \left[ \alpha_{ii}^2 \| \betals_{ii} \|^2_F + \sum_{j \in \sfL} \alpha^2_{ij} \| \betals_{ij}\|_F^2 + \sum_{j \in \sfL} \alpha^2_{ji} \|\betals_{ji}\|_F^2  \right]. \\
                                           & \geq   \frac{ \inprod{\nabla \ell (\betal)}{\betals_{ii}}^2  }{2\Mo \| \betals_{ii}\|_F^2} + \sum_{j \in \sfL} \left( \frac{ \inprod{\nabla \ell (\betal)}{\betals_{ij}}^2  }{2\Mo \| \betals_{ij}\|_F^2} + \frac{ \inprod{\nabla \ell (\betal)}{\betals_{ji}}^2  }{2\Mo \| \betals_{ji}\|_F^2}  \right),
\end{align*}

where the last inequality follows by setting $\alpha_{ij} = \frac{\inprod{\nabla \ell (\betal)}{\betals_{ij}} }{\Mo \| \betals_{ij}\|_F^2} $ for $j \in \sfL$, and for $j=i$.

Summing up for all $i \in \sfS$, we get 

\begin{align*}
\sum_{i \in \sfS}  \ell(\bB^{(\sfL \cup \{i\})}) - \ell(\betal) & \geq \sum_{i \in \sfS} \left[ \frac{ \inprod{\nabla \ell (\betal)}{\betals_{ii}}^2  }{2\Mo \| \betals_{ii}\|_F^2} + \sum_{j \in \sfL} \left( \frac{ \inprod{\nabla \ell (\betal)}{\betals_{ij}}^2  }{2\Mo \| \betals_{ij}\|_F^2} + \frac{ \inprod{\nabla \ell (\betal)}{\betals_{ji}}^2  }{2\Mo \| \betals_{ji}\|_F^2}  \right) \right] \\
& = \frac{1}{2\Mo} \| P_{\bU_\sfS} \nabla \ell( \betal) P_{\bV_\sfS} \|^2_F
\end{align*}
\end{proof}

\subsection{Proofs for greedy improvement}
 Let $\sfS_i^G$ be the support set formed by Algorithm~\ref{alg:greedy} at iteration $i$. Define $A(i) := f(\sfS_i^G) - f(\sfS_{i-1}^G)$ with $A(0) = 0$ as the greedy improvement. We also define $B(i) := f(\sfS^*) -
f(\sfS_i^G)$ to be the remaining amount to improve, where $\sfS^\star$ is the optimum $k$-sized solution. We provide an auxiliary Lemma that uses the submodularity ratio to lower bound the greedy improvement in terms of best possible improvement from step $i$.
\begin{lemma}
\label{lem:greedyimprove}
  At iteration $i$, the incremental gain of the greedy method (Algorithm~\ref{alg:greedy}) is
  \begin{equation*}
    A(i+1) \geq \frac{\tau \gamma_{\sfS_i^G,r}}{r} B(i).
  \end{equation*}
\end{lemma}
\begin{proof}
Let $\sfS = \sfS_i^G$. Let $\sfS^R$ be the sequential
  orthogonalization of the atoms in $\sfS^*$ relative to $\sfS$. Thus,
  \begin{align*}
    rA(i+1) \geq |\sfS^R| A(i+1) & \geq \tau |\sfS^R| \max_{a \in \sfS^R} f(\sfS \cup \{a\}) - f(\sfS) \\
                                 & \geq \tau  \sum_{a \in \sfS^R} [f(\sfS \cup \{a\}) - f(\sfS)] \\
                                 & \geq \tau \gamma_{\sfS,|\sfS^R|} [ f(\sfS \cup \sfS^R) - f(\sfS) ]\\
                                  & \geq \tau \gamma_{\sfS,|\sfS^R|} B(i)
  \end{align*}
  Note that the last inequality follows because $f(\sfS \cup \sfS^R) \geq f(\sfS^*)$. The
  penultimate inequality follows by the definition of weak submodularity, which applies in this case
  because the atoms in $\sfS^R$ are orthogonal to eachother and are also orthogonal to $\sfS$.
\end{proof}

Using Lemma~\ref{lem:greedyimprove}, one can prove an approximation guarantee for Algorithm~\ref{alg:greedy}. 

\subsubsection{Proof of Theorem~\ref{thm:approxGreedy}}

\begin{proof}
From the notation used for Lemma~\ref{lem:greedyimprove}, $A(i+1) = B(i) - B(i+1)$. Let $C=\frac{\tau \gamma_{\sfS^G_i,r}}{r}$. From Lemma~\ref{lem:greedyimprove}, we have, 

\begin{align*}
B(i+1) \leq (1 - C) B(i) \leq (1-C)^{i+1} B(0).
\end{align*}

From its definition, $B(0) = f(\sfS^\star) - f(\emptyset)$. So we get, 

\begin{align*}
 & \left[f(\sfS^\star) - f(\emptyset) \right] - \left[f(\sfS^G_i) - f(\emptyset) \right] \leq  (1 - C)^i  \left[ f(\sfS^\star) - f(\emptyset) \right] \\
\implies & \left[f(\sfS^G_i) - f(\emptyset) \right]  \geq  ( 1 - (1-C)^i)  \left[f(\sfS^\star) - f(\emptyset) \right] \geq \left( 1 - \frac{1}{e^  {\tau \gamma_{\sfS^G_i,r} \frac{k}{r}}} \right) \left[ f(\sfS^\star) - f(\emptyset) \right]
\end{align*}

from which the result follows.
\end{proof}

%\section{Strong convexity assumptions by~\citet{ShalevShwartz:2011vi}}

 \subsection{Proof for GECO bounds}
 
  Let $\sfS^O_i$ be the support set selected by the GECO procedure (Algorithm~\ref{alg:omp}) at iteration $i$. Similar to the section on greedy improvement, we define some notation. Let $D(i) := f(\sfS^O_i) - f(\sfS^O_{i-1})$ be the improvement made at step $i$, and as before we have $B(i)= f(\sfS^\star) - f(\sfS^O_i)$ be the remaining amount to improve.

We prove the following auxiliary lemma which lower bounds the gain after adding the atom selected by the subroutine \OMPsel in terms of operator norm of the gradient of the current iterate and smoothness of the function. 
\begin{lemma}
\label{lem:ompimprove}
Assume that $\ell(\cdot)$ is $m_{i}$-strongly concave and $M_{i}$-smooth over matrices of in the set $\tilde{\Omega}:= \{ (\bX,\bY): \text{rank}(\bX - \bY) \leq 1 \} $. Then, 
\begin{equation*}
D(i+1) \geq \frac{\tau m_{r+k}}{r \Mo}   B(i).
\end{equation*}
\end{lemma}
\begin{proof}
For simplicity, say $\sfL = \sfS^O_i$. Recall that for a given support set $\sfL$, $f(\sfL) = \ell(\betal)$ \textit{i.e.} we denote by $\betal$ the argmax for $\ell(\cdot)$ for a given support set $\sfL$. Hence, by the optimality of $\bB^{(\sfL \cup \{i\})}$,
\begin{align*}
D(i+1) &  = \ell(\bB^{(\sfL \cup \{i\})}) - \ell (\betal) \\
& \geq \ell( \bB^{(\sfL)} + \alpha \bu \bv^\top) - \ell(\betal) 
\end{align*}

for an arbitrary $\alpha \in \bbR$, and the vectors $\bu,\bv$ selected by \OMPsel. Using the smoothness of the $\ell(\cdot)$, we get,

\begin{align*}
D(i+1) \geq  \alpha \inprod{\nabla \ell(\betal)}{ \bu \bv^\top } - \alpha^2 \frac{\Mo}{2}
\end{align*}

Putting in $\alpha = \frac{\tau}{\Mo}\| \nabla \ell(\betal)\|_2$, and by $\tau$-optimality of \OMPsel, we get, 
\begin{align*}
D(i+1) \geq  \frac{\tau^2}{2\Mo}\| \nabla \ell(\betal)\|_2^2
\end{align*}

Let $\sfS^R$ be obtained from after sequentially orthogonalizing $\sfS^\star$ w.r.t. $\sfS_i$. By definition of the operator norm, we further get,

\begin{align*}
D(i+1)  &  \geq  \frac{\tau^2}{2\Mo}\| \nabla \ell(\betal)\|_2^2 \\
& \geq  \frac{\tau^2}{2 r \Mo} \sum_{i \in \sfS^R} \inprod{\bu_i \bv_i^\top}{ \nabla \ell(\betal) }^2 \\
& = \| P_{\bU_{\sfS^R}} \nabla \ell(\betal)  P_{\bV_{\sfS^R}} \|_F^2\\
& \geq \frac{\tau^2 m_{r+k}}{r \Mo} \left( \ell(\bB^{\sfL \cup \sfS^R}) - \ell(\betal) \right)\\
& \geq \frac{\tau^2 m_{r+k}}{r \Mo} \left( \ell(\bB^{\sfS^\star}) - \ell(\betal) \right)\\
& =  \frac{\tau^2 m_{r+k}}{r \Mo}  B(i)
\end{align*}
\end{proof}

The proof for Theorem~\ref{thm:approxOMP} from Lemma~\ref{lem:ompimprove} now follows using the same steps as for Theorem~\ref{thm:approxGreedy} from Lemma~\ref{lem:ompimprove}. 

\subsection{Proof for recovery bounds}

\subsubsection{Proof of Theorem~\ref{thm:recovery}}

For clarity of representation, let $C = C_{r,k}$, and for an arbitrary $\bH \in \bbR^{r \times r}$, let $\bB_\sfr = \bU_\sfS^\top \bH \bV_\sfS $, and $\bDelta:= \bB^{(\sfS_r)} - \bB_\sfs$. Note that $\bDelta$ has rank atmost $(k+r)$. Recall that by the $m_{k+r}$ RSC (Definition~\ref{def:RSCRSM}), 

\begin{align*}
\ell(\bB^{(\sfS_k)}) - \ell (\bB_\sfr) - \inprod{\nabla \ell (\bB_\sfr)}{\bDelta} \leq \frac{-m_{k+r}}{2}{\| \bDelta \|_F^2 }.
\end{align*}

From the approximation guarantee, we have, 

\begin{align*}
& \ell(\bB^{(\sfS_k)} ) - \ell(\bB_\sfr) \geq  (1 - C) [ \ell(\mathbf{0)} - \ell(\bB_\sfr)] \\
\implies & \ell(\bB^{(\sfS_k)} ) - \ell(\bB_\sfr)  - \inprod{\nabla \ell (\bB_\sfr)}{\bDelta}   \geq  (1 - C) [ \ell(\mathbf{0}) - \ell(\bB_\sfr)]  - \inprod{\nabla \ell (\bB_\sfr)}{\bDelta} \\
\implies & \frac{-m_{k+r}}{2}{\| \bDelta \|_F^2 }   \geq  (1 - C) [ \ell(\mathbf{0}) - \ell(\bB_\sfr)]  - \inprod{\nabla \ell (\bB_\sfr)}{\bDelta}\\
&\qquad\qquad\qquad \geq (1 - C) [ \ell(\mathbf{0}) - \ell(\bB_\sfr)]  - (k+r)^{\nicefrac{1}{2}}    \| \nabla \ell (\bB_\sfr)  \|_2  \|\bDelta \|_F, 
\end{align*}

where the last inequality is due to generalized Holder's inequality. Using $2ab \leq ca^2 + \frac{b^2}{c}$ for any positive numbers $a,b,c$, we get 

\begin{align*}
\frac{m_{k+r}}{2}{\| \bDelta \|_F^2 }    \leq    (k+r)  \frac{ \| \nabla \ell (\bB_\sfr)  \|^2_2}{m_{k+r}} + \frac{m_{k+r} \|\bDelta \|^2_F}{4} + (1 - C) [  \ell(\bB_\sfr)- \ell(\mathbf{0})] ,
\end{align*}

which completes the proof.

}
\end{document}